\documentclass[letterpaper, journal, twoside]{ieeeconf}
\usepackage{times}

\usepackage{multicol}

\usepackage{comment}
\usepackage{soul}

\usepackage[bookmarks=true]{hyperref}

\renewcommand{\baselinestretch}{0.923}

\usepackage{graphics} 
\usepackage{epsfig} 
\usepackage{times}

\usepackage{amsmath} 
\usepackage{amssymb} 
\usepackage{bm}  
\usepackage{siunitx} 
\usepackage{textcomp}
\usepackage{gensymb}
\usepackage{booktabs}
\usepackage{multirow}
\usepackage{xcolor}
\usepackage{cleveref}
\let\labelindent\relax
\usepackage{enumitem}
\usepackage{xspace}
\usepackage{amsfonts}
\usepackage{mathrsfs}
\usepackage{mathtools}
\usepackage[ruled,vlined,linesnumbered]{algorithm2e}
\usepackage{makecell}
\usepackage{multicol}
\usepackage{rotating}
\usepackage{tikz}
\usepackage{courier}

\usepackage{caption}
\usepackage{subcaption}

\usepackage{graphicx} 
\usepackage{pgfplots} 
\pgfplotsset{compat=1.18} 

\definecolor{CranJ}{cmyk}{0,0.69,0.54,0.04} 
\definecolor{PinkJ}{cmyk}{0,0.71,0.43,0.12} 
\definecolor{Cran}{cmyk}{0,0.73,0.41,0.29} 
\definecolor{VRed}{cmyk}{0,0.75,0.25,0.2} 
\definecolor{ORed}{cmyk}{0,0.75,0.75,0} 
\definecolor{CBlue}{cmyk}{1,0.25,0,0} 
\usepackage{bm}

\newcommand{\vect}[1]{\boldsymbol{\mathbf{#1}}}

 \newcommand{\boxend}{\hfill \ensuremath{\Box}}
\newcommand{\oprocendsymbol}{\hbox{$\bullet$}}
\newcommand{\oprocend}{\relax\ifmmode\else\unskip\hfill\fi\oprocendsymbol}

\allowdisplaybreaks

\newtheorem{thm}{Theorem}[section]

 \newtheorem{lem}{Lemma}
\newtheorem{assump}{Assumption}

\definecolor{lightred}{rgb}{1.0, 0.7, 0.7}
\definecolor{darkred}{rgb}{0.5, 0.0, 0.0}
\definecolor{lightblue}{rgb}{0.7, 0.7, 1.0}
\definecolor{darkblue}{rgb}{0.0, 0.0, 0.7}

\NewDocumentCommand{\todo}{o m}{\textcolor{red}{\textbf{TODO\IfNoValueTF{#1}{}{(#1)}:} #2}}
\NewDocumentCommand{\note}{o m}{\textcolor{orange}{\textbf{NOTE\IfNoValueTF{#1}{}{(#1)}:} #2}}

\newcommand{\longthmtitle}[1]{\mbox{}\textit{{(#1):}}}
\begin{document}

\title{BEASST: Behavioral Entropic Gradient based Adaptive \\ Source Seeking for Mobile Robots
}



\author{Donipolo Ghimire${^1}$~~ Aamodh Suresh${^2}$~~ Carlos Nieto-Granda${^2}$ ~~ Solmaz S. Kia${^1}$, \emph{Senior Member, IEEE}
\thanks{Manuscript received: August 10, 2025; Revised November 2, 2025; Accepted November 28, 2025. This paper was recommended for publication by Editor Olivier Stasse upon evaluation of the Associate Editor and Reviewers' comments.}
\thanks{${^1}$D.~Ghimire and S.~Kia are with the Department of Mechanical and Aerospace Engineering, University of California, Irvine, CA 92697 {\tt\footnotesize dghimire,solmaz@uci.edu}} 
\thanks{${^2}$ A.~Suresh and C.~Nieto-Granda are with the U.S. DEVCOM Army Research Laboratory (ARL), Adelphi, MD 20783 {\tt\footnotesize,aamodh@gmail.com, carlos.p.nieto2.civ@army.mil}}
}

\maketitle

\begin{abstract}
This paper presents BEASST (Behavioral Entropic Gradient-based Adaptive Source Seeking for Mobile Robots), a novel framework for robotic source seeking in complex, unknown environments. Our approach enables mobile robots to efficiently balance exploration and exploitation by modeling normalized signal strength as a surrogate probability of source location. Building on Behavioral Entropy (BE) with Prelec's probability weighting function, we define an objective function that adapts robot behavior from risk-averse to risk-seeking based on signal reliability and mission urgency. The framework provides theoretical convergence guarantees under unimodal signal assumptions and practical stability under bounded disturbances. Experimental validation across DARPA SubT and multi-room scenarios demonstrates that BEASST consistently outperforms state-of-the-art methods and exhibits strong robustness to noisy gradient estimates while maintaining convergence. BEASST achieved 15\% reduction in path length and 20\% faster source localization through intelligent uncertainty-driven navigation that dynamically transitions between aggressive pursuit and cautious exploration. 
\end{abstract}

\section{Introduction}
\vspace{-0.05in}
This paper addresses active source seeking (SS) in unknown environments via a mobile robot. SS aims to move the robot to a target location based on perceived information from a signal of interest (SoI), such as radio signals, radiation sources, or plumes from hazardous agents~\cite{JF-VK:10, FM-TW-CP-KA:18, AF-SJ:22, GC-KA-PT-AL:17}. These signals are often affected by environmental obstacles and distance, making target localization challenging across applications including environmental monitoring, search and rescue operations, and hazardous material detection. Traditionally, SS methods often rely on extremum seeking controllers~\cite{SA-GP:12}, flow field tracking~\cite{SL-YG-BB:14}, or employ simplified probabilistic models that assume static uncertainty where every measurement is treated equally trustworthy~\cite{JF-VK:10,CD-OP-JO-AA:23}.
Extremum-seeking approaches, while providing real-time capabilities, suffer from vanishing gradients in far-field regions and overshooting behavior near sources, limiting their effectiveness in complex signal landscapes. Alternative approaches include bio-inspired methods such as infotaxis~\cite{vergassola2007infotaxis} for odor source localization, which exploit intermittent signal detection patterns but often lack theoretical convergence guarantees. Multi-robot cooperative strategies~\cite{SL-RK-YG:14,spears2006distributed, clark2006cooperative} have demonstrated promise in distributed scenarios but introduce coordination complexity and communication overhead that can limit scalability. Similarly, discrete planning approaches evaluate information gain at the boundary of known and unknown areas (frontiers), greedily selecting them until sources are found~\cite{VJ-DG:12,CN-JR-CH:21}. While entropy-based frameworks like Infosploration~\cite{CN-JR-CH:21} use Shannon entropy (SE) to model environmental uncertainty for strategic trajectory planning, they often struggle with global optimality and efficient information gathering in complex environments with varying signal characteristics. The fundamental limitation lies in their uniform treatment of uncertainty across different signal strengths, failing to capture the nuanced decision-making required for effective source seeking. A critical limitation of these conventional methods is their inability to dynamically adapt motion strategies based on perceived signal strength, leading to vanishing gradients or overshooting. Furthermore, diverse SS applications demand flexible policies that can respond to varying environmental conditions and signal characteristics, often requiring extensive manual parameter tuning to balance exploration and exploitation effectively.

\begin{figure}[t]
    \centering
    \includegraphics[width=0.45\textwidth]{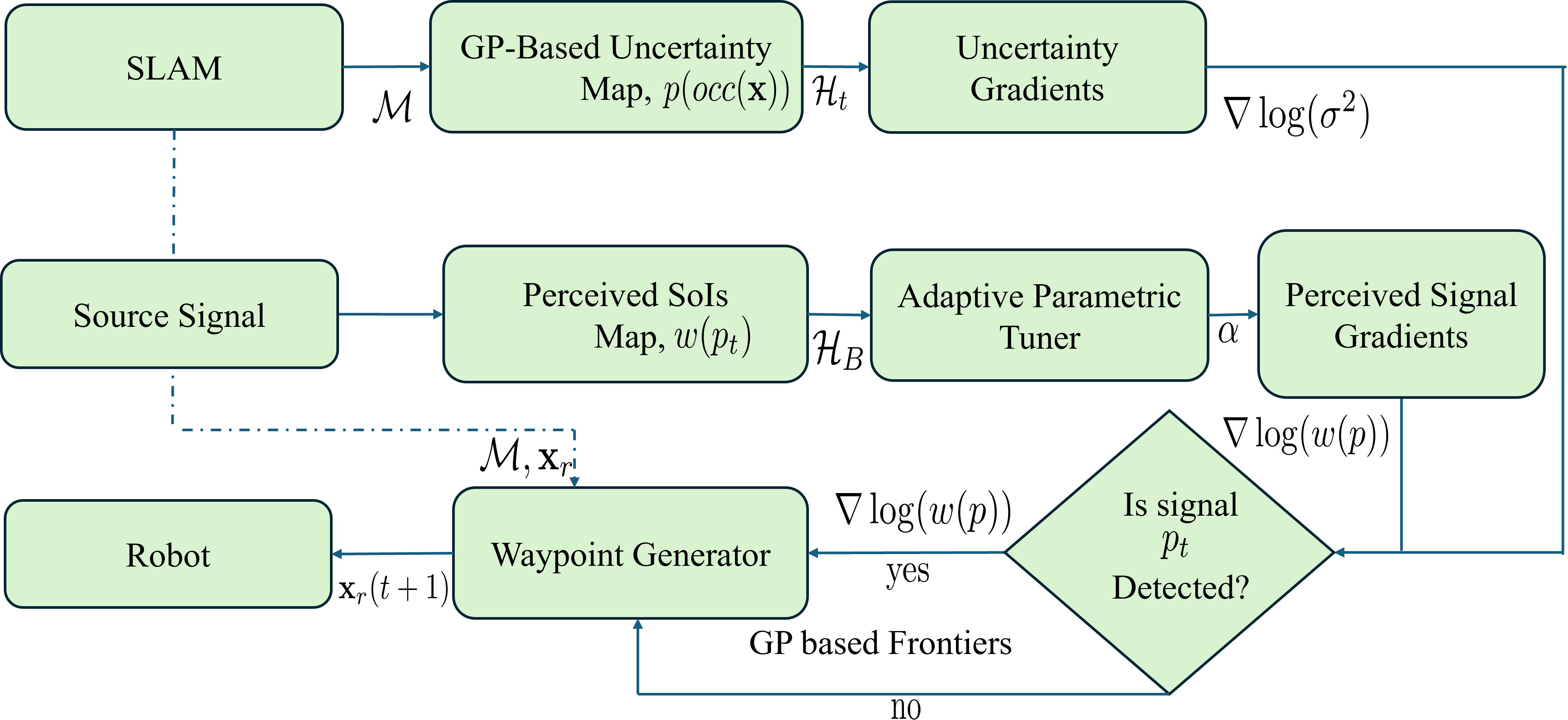} 
    \caption{\small{ A schematic overview of the proposed system for Source-Seeking. The SLAM module provides the occupancy grid map $\mathcal{M}$, and the robot's localization $\vect{x}_r$. $p(occ(\mathbf{x}))$ is the occupancy probability computed using Gaussian Process (GP). The differential entropy of the map at $\vect{x}_r$, $\mathcal{H}_t$ is evaluated. The entropy is simplified and its gradients are calculated (See~\ref{sec::BEASST_framework}). When a signal is detected, the perceived SoI, $w(p_{t})$, is computed using a source modeled as a normalized signal strength field and a weighting function. We then compute the Behavioral Entropy $\mathcal{H}_B$ for source seeking and use an Adaptive Parametric Tuner to select the appropriate $\alpha$. If the signal is not detected, we explore the map using GP-based frontiers. The waypoint generator produces the waypoints used by the controller to guide the robot toward the source (See~\ref{sec::BEASST_framework}).
    }}
    \label{fig::RAL_BEASST_prob_formulation}
    \vspace{-0.2cm}
\end{figure}

To overcome these limitations, Behavioral Entropy (BE)~\cite{AS-CN-SM:24} offers avenues for adaptable uncertainty perception through its foundation in prospect theory and cognitive decision-making models. BE stands out for its intuitive encoding of human-like cognitive biases in uncertainty perception, offering richer means to influence robot behavior~\cite{AS-SM:21,AS-CN-SM:24} through probability weighting functions~\cite{DP:98}. This framework can model both risk-averse and risk-seeking behaviors. Its scalar parameters ($\alpha$ and $\beta$) provide a powerful tuning mechanism, enabling the robot to dynamically tailor its strategy from risk-averse to risk-seeking based on varying signal strengths and mission requirements.~\cite{AS-CN-SM:24} and~\cite{WAS-AS-CN:25} used BE in exploration objectives and showed improved performance in exploration and data generation, motivating BE's extension to SS by providing naturally adaptive behaviors. 

To address these fundamental limitations, we propose BEASST, a novel continuous control framework that seamlessly integrates exploration and exploitation phases. While leveraging GP-based exploration~\cite{MA-HJ-NR-LL:23} for environmental understanding, this work primarily advances the SS methodology through innovative behavioral modeling. BEASST employs a data-driven methodology that models normalized signal strength as a surrogate probability for source location, enabling direct application of probabilistic methods without requiring explicit signal propagation models. Building on BE, this framework defines an objective function over this surrogate probability, generating dynamic source-seeking velocity field via a 2-Wasserstein gradient flow that adapts to local signal characteristics. These fields enable autonomous adaptation of search behavior, transitioning smoothly between aggressive exploration in weak-signal regions and cautious exploitation near potential sources. This adaptive capability is crucial for enhancing SS performance in weak signal scenarios where traditional methods often fail. The overall framework is depicted in Fig.~\ref{fig::RAL_BEASST_prob_formulation}  and our primary contributions are:
\begin{itemize}
\item \textit{Data-Driven Surrogate Probability for SoI:} Modeling normalized signal strength as a surrogate spatial probability, which enables computationally efficient and data-driven SS without requiring signal propagation models.

\item \textit{Behavioral Entropic Gradients (BUGs):} A novel SS control framework via localized differential BE, formulated from 2-Wasserstein gradient flow to yield adaptive source-seeking velocity fields with theoretical convergence and stability guarantees under bounded Gaussian noise.

\item \textit{BEASST Framework:} Real-time adaptive behavior tuning that autonomously refines control policies, enabling smooth transitions between aggressive exploration and cautious exploitation for expedited convergence, particularly effective in weak-signal scenarios.

\end{itemize}
We validate our framework through several ROS-Unity simulation environments with different configurations and several comparisons with frontier-based SS methods like Infosploration~\cite{CN-JR-CH:21}, Behaviorsploration~\footnote{Is an extension of Infosploration by incorporating BE to evaluate information gain of frontiers during the exploration phase.} and random frontiers. Moreover, we show BEASST exhibits strong robustness to noisy gradients in these environments, and maintains convergence and performance under bounded Gaussian disturbances consistent with our theoretical stability guarantees.


\section{Problem Setting and Preliminaries}
\label{sec::RAL_BEASST_problem_statement}
\vspace{-0.05in}
This paper addresses source seeking in a bounded, unknown $2$D environment $\mathcal{E} \subset \mathbb{R}^2$ with a single mobile robot. Equipped with an exteroceptive sensor (e.g., Lidar), the robot uses SLAM~\cite{ST:05} for localization ($\mathbf{x}_{r}(t) \in SE(2)$) and mapping ($\mathcal{M}_D$) in $\mathcal{E}$. Its primary objective is to locate $N_{s}\in\mathbb{Z}_{\geq1}$ accessible sources within $\mathcal{E}$ (i.e., reachable via traversable regions) by utilizing their emitted signals of interest (SoIs). At each point $\vect{x}\in\mathcal{E}$, the robot measures SoI strength, denoted as $P_{dBm}(\vect{x})$.

\begin{assump}\longthmtitle{Distinguishable and Non-Interfering Signals}
Signals from distinct sources are distinguishable (e.g., unique identifiers like MAC addresses in wireless communication) and non-interfering, allowing for independent and accurate measurement of each source's strength.\boxend
\end{assump}

To effectively discover and detect SoI, the robot employs a frontier-based exploration strategy. We specifically utilize GP-based frontier exploration~\cite{MA-HJ-NR-LL:23} due to its enhanced performance, stemming from its ability to provide a predictive model of unknown areas. This algorithm is briefly reviewed in Section~\ref{sec::BEASST_framework}, detailing its integration into our SS framework. Upon SoI detection during exploration, the robot transitions to a dedicated SS mode. Subsequent sections detail our novel SS methodology and the adaptive switching mechanism within the BEASST framework. To develop our SS we rely on innovative use of BE to develop a merit function whose minimization guides the robots towards the detected signal source. BE~\cite{AS-CN-SM:24} quantifies the subjective perception of uncertainty, drawing inspiration from human decision-making models in Behavioral Economics~\cite{DP:98, SD:16}. Its differential form, for a pdf $p(\mathbf{x})$, is given by:
\begin{equation}
\mathcal{H}_B(p) = -\int w(p(\vect{x})) \log(w(p(\vect{x}))) \, \text{d}\vect{x}.
\label{eqn:be_differential}
\end{equation}
Central to BE is Prelec's probability weighting function~\cite{DP:98}, $w : [0, 1] \rightarrow [0, 1]$, defined as:
$w(p) = \textup{e}^{-\beta(-\log (p))^\alpha},$
where $p$ is an `objective' uncertainty measure.\footnote{For the differential form of BE, where $p(\vect{x})$ is a probability density function, special consideration is often required as $p(\vect{x})$ can take values outside $[0,1]$~\cite{WAS-AS-CN:25}. In our application, however, $p(\vect{x})$ represents a normalized surrogate probability derived from signal strength, which by definition lies within the $[0,1]$ range, ensuring compatibility with $w(p)$.} When $\alpha =1$ and $\beta=1$, $\mathcal{H}_B$ equates to Shannon entropy. The parameter $\alpha > 0$ controls the convexity or concavity of $w$, significantly altering perceived probability: a small $\alpha \approx 0$ makes $w$ concave and over-weights probabilities ($w(p) \gg p$), while $\alpha \to \infty$ under-weights them ($w(p) \ll p$). The parameter $\beta > 0$ scales the overall weighting~\cite{AS-CN-SM:24}. Thus, Prelec's function $w$ encodes the robot's `subjective perception' of uncertainty, making $w(p)$ the \emph{perceived} uncertainty measure. BE's parameters ($\alpha, \beta$) can dynamically change the robot's behavior by altering its perceived notion of uncertainty. For instance, a small $\alpha$ can magnify perceived uncertainty in regions of low objective probability, making the robot more uncertainty-averse in those areas.

\section{Behavioral Source Seeking}
\label{sec:source_seeking}
\vspace{-0.05in}
In SS, generally, moving in the direction of $\nabla P_{dBm}(\vect{x})$ (the gradient of the raw signal strength) is a valid, simple approach for reaching the source. However, this simple gradient ascent might suffer from issues like vanishing gradients far from the source (slow progress) or overshooting near the source (oscillations). It also doesn't allow for sophisticated control over the perception of the signal. In what follows, we develop a behavioral SS approach, which using the concept of BE allows the robot to adjust its perceived signal strength and control its gradient size and thereby avoiding the limitations associated with purely gradient-following methods.

We define the normalized measured signal strength as:
\begin{equation}
\label{eq::normal_Pdm}
p(\vect{x})=\phi({P_{dBm}(\vect{x})}),  \quad \phi:\mathbb{R} \rightarrow (0,1]
\end{equation} where $P_{dBm}(\vect{x})$ is the detected signal strength at $\vect{x}$, where $\phi(\cdot)$ is a continuously differentiable and strictly increasing normalizing function. This general formulation encompasses various normalizations, where one could choose $(\phi(s) = \frac{s}{Z})$, $(\phi(s)=s / P_{dBm}^{\max})$ or any increasing function. In our SS, this normalized signal strength function serves as a \emph{surrogate probability density function}. Its value is almost $1$ at the source location and decreases with distance, reflecting the diminishing signal~strength. Our rigorous theoretical development presented below relies on the following assumptions regarding the characteristics of the signal strength and signal reception. It is worth noting that empirical studies carried out in Section~\ref{sec::Simulation and Results} indicate the method's robust performance in complex real-world scenarios where obstacles and noise may lead to localized deviations from ideal continuity conditions.

\begin{assump}\longthmtitle{Single-Source Signal Profile (Unimodal and Differentiable)}\label{assump::Unimodal}
For each individual signal source, the raw signal strength, $P_{\text{dBm}}(\vect{x})$, is a \emph{continuous and differentiable function} of the robot's position $\vect{x}$. Furthermore, $P_{\text{dBm}}(\vect{x})$ is \emph{unimodal} and possesses a \emph{single global maximum}, which unambiguously defines the unique source location $\vect{x}_s$ of its emitting source. As the distance from $\vect{x}_s$ increases, $P_{\text{dBm}}(\vect{x})$ continuously \emph{diminishes} ($P_{\text{dBm}}(\vect{x}) \to 0$) as $\|\vect{x}\| \to \infty$. This implies that the normalized signal strength, defined in~\eqref{eq::normal_Pdm}, is unimodal with a unique global maximum of $1$ at $\vect{x}_s$, and $p(\vect{x}) \to 0$ as $\|\vect{x}\| \to \infty$.\boxend
\end{assump}

\section{Behavioral motion policy design for source~seeking}
\vspace{-0.05in}
\begin{figure}[t]
    \centering
    \includegraphics[width=0.7\linewidth]{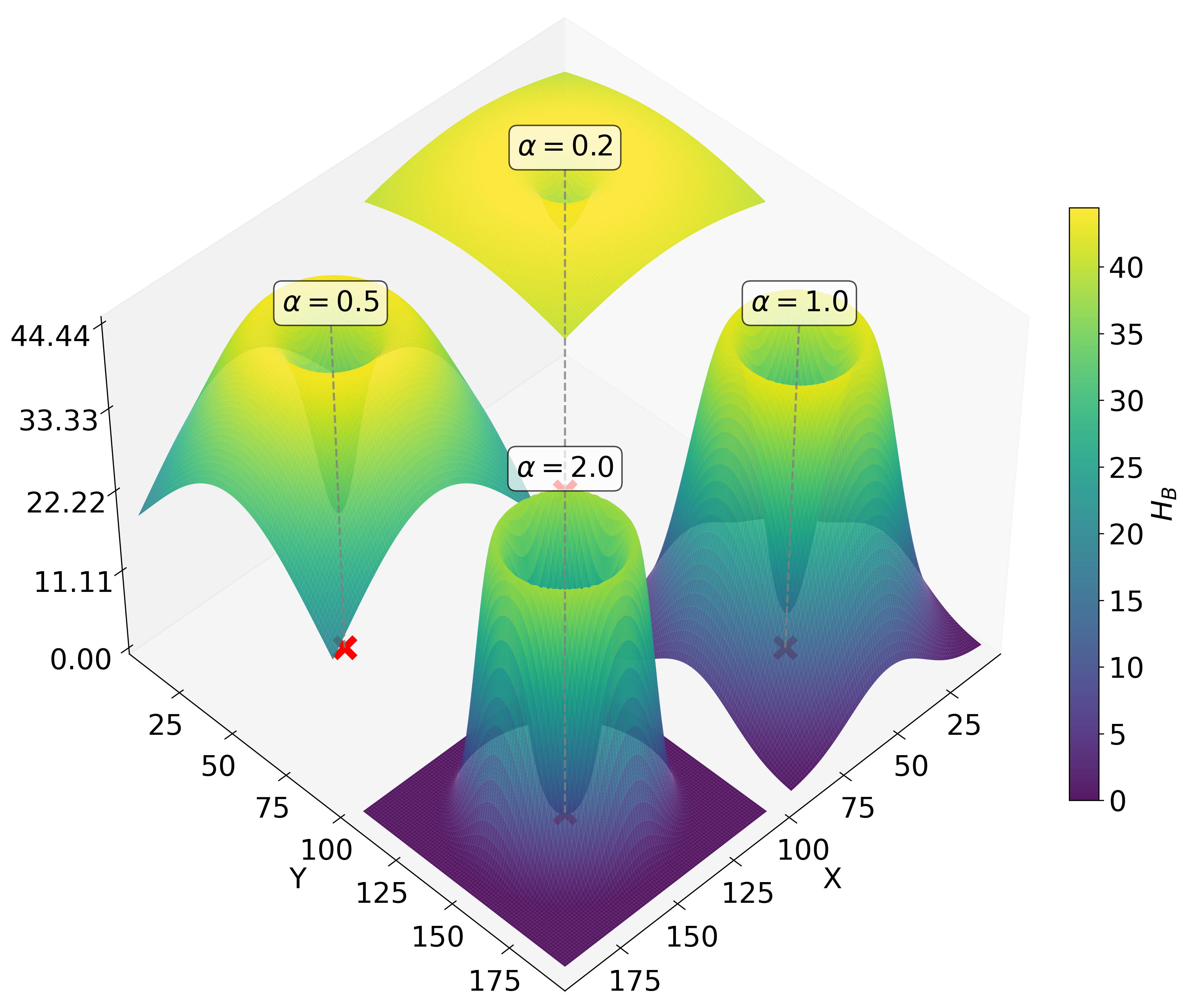}
    \caption{\small Localized differential BE~\eqref{eq::behavioral_entropy} computed over a grid in local patches for an exponentially decaying signal‐strength model (See ~\eqref{eq::SOI_model}) for $\beta=1$ and $\alpha\in \{ 0.2, 0.5, 1.5 ,2.0\}$ with $4$ sources located at \{(50,50), (150,50), (50,150), (150,150)\}.  }
    \label{fig:RAL_BEASST_local_HB}
\end{figure}

To guide the robot towards the source using the normalized measured signal strength $p(\vect{x})$ defined in~\eqref{eq::normal_Pdm}, we propose to continuously minimize its local \emph{differential BE}, $\mathcal{H}_B$, defined using prelec's weighting function (fixing $\beta = 1$) with one time-varying free parameter $\alpha_{t}$:
\begin{align}
\label{eq::behavioral_entropy}
\mathcal{H}_B(t,\vect{x}_r(t))=-\int_{\mathcal{R}_{\vect{x}_{r}(t)}} \!\!\!w(p(\vect{x}))
\log\big(w(p(\vect{x}))\big) \, \mathrm{d}\vect{x},
\end{align}
where $\mathcal{R}_{\vect{x}_{r}(t)}$ is a closed and compact area around robot location $\vect{x}_{r}(t)$. For notational brevity, we denote $w(p(\vect{x}), \alpha_t)$ as $w(p(\vect{x})$. However, this minimization is not carried out in the conventional Euclidean sense of finding the lowest point on a scalar landscape, as directly differentiating $\mathcal{H}_B$ with dynamically changing integration region $\mathcal{R}_{\mathbf{x}_r(t)}$ is challenging and may lead to undesirable local minima in regions of low signal strength. Instead, we seek a procedure similar to the steepest descent of $\mathcal{H}_B$ within the 2-Wasserstein space of probability distributions by considering variations in the underlying perceived signal density $w(p(\mathbf{x}))$. This approach uniquely defines a velocity field that guides the robot by using perceived signal density towards regions of higher concentration and magnitude, thereby ensuring convergence to the source. Fig.~\ref{fig:RAL_BEASST_local_HB} illustrates BE computed over a grid in local patches and clearly shows how different values of $\alpha$ reshape the entropy landscape, thereby influencing the robot's trajectory towards the source. Framing $\mathcal{H}_B$ in this form also opens doors to analytical manipulation using functional derivatives, as shown later in this section.

\begin{figure}[t]
    \centering 
\begin{subfigure}[t]{0.3\columnwidth}
   \centering 
\includegraphics[width=0.9\linewidth]{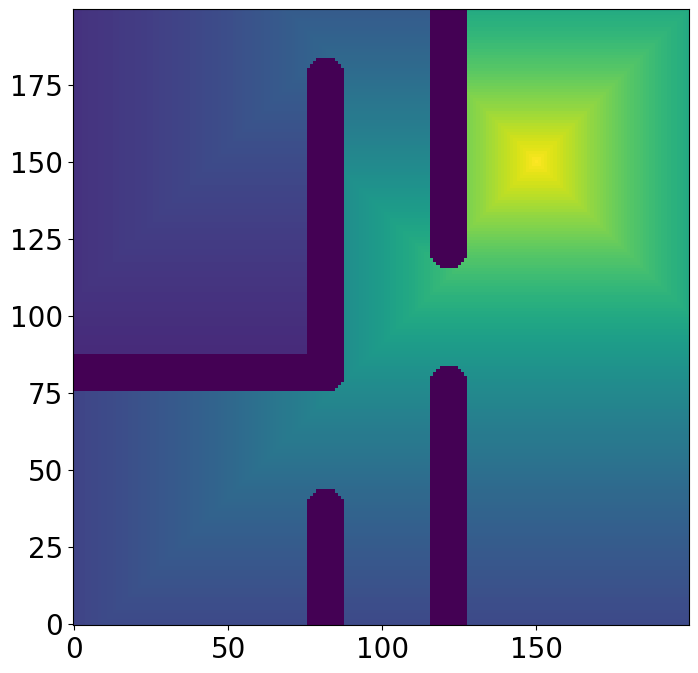}
  \caption{{\small $w(p)$, $\alpha=1.0$}}
  \label{subfig2:alpha1}
\end{subfigure}\hfil ~
\begin{subfigure}[t]{0.3\columnwidth}
   \centering 
  \includegraphics[width=0.9\linewidth]{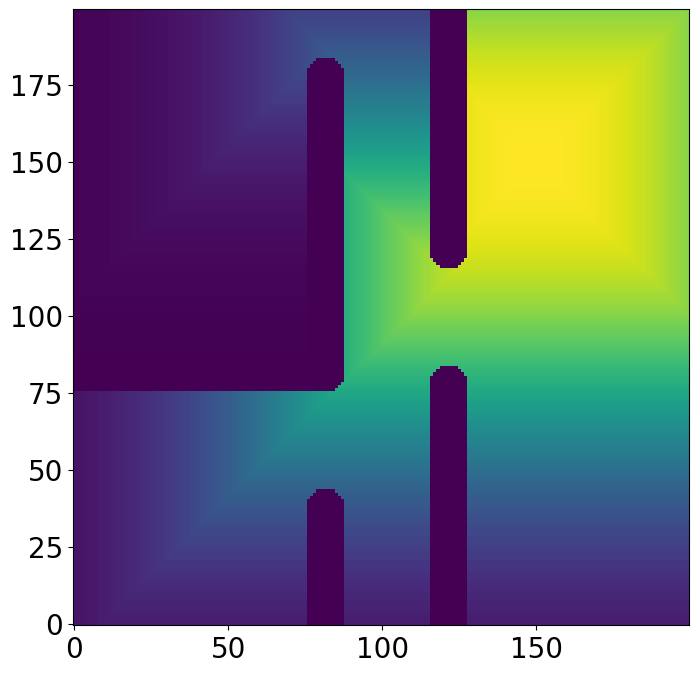}
  \caption{{\small $w(p)$, $\alpha=2.0$}}
  \label{subfig2:alpha2}
\end{subfigure}\hfil ~
\begin{subfigure}[t]{0.3\columnwidth}
  \centering 
  \includegraphics[width=0.9\linewidth]{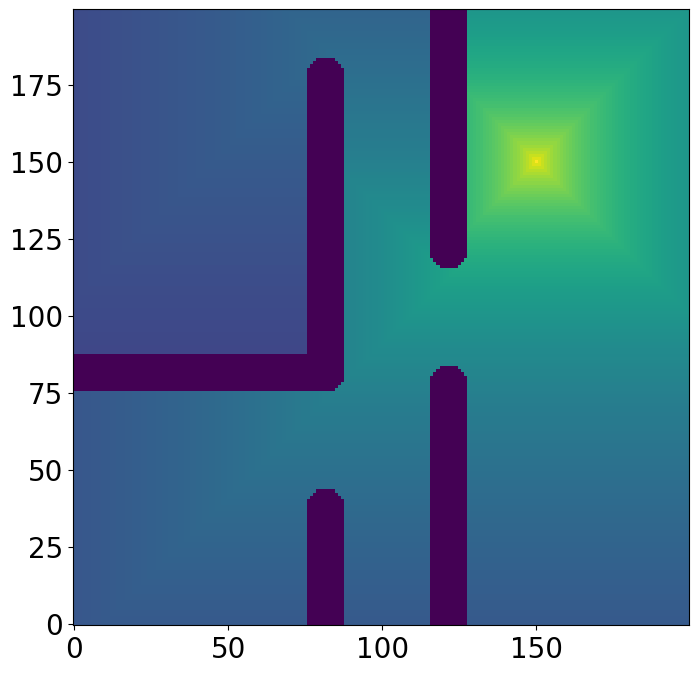}
  \caption{$w(p)$, $\alpha=0.8$}
  \label{subfig2:alpha0.8}
\end{subfigure} 
\medskip
\begin{subfigure}[t]{0.3\columnwidth}
   \centering 
  \includegraphics[width=\linewidth]{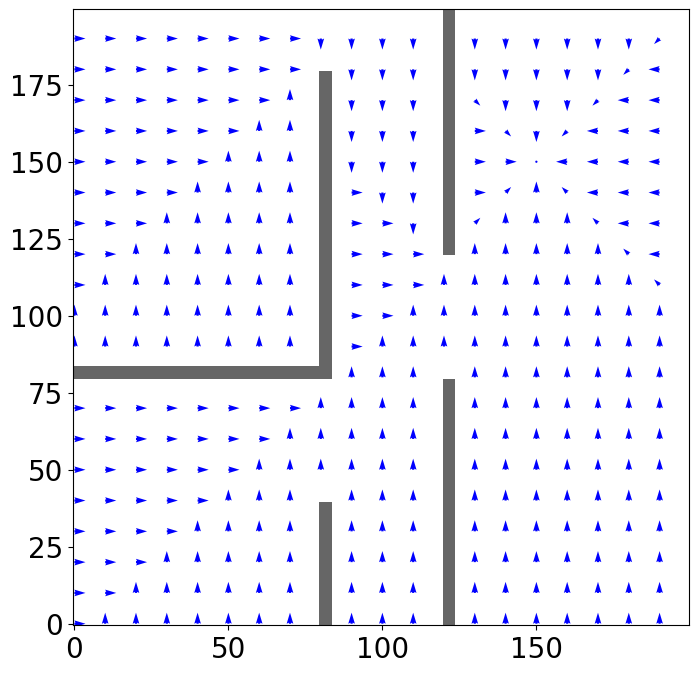}
  \vspace{-0.14in}
  \caption{$\nabla \log( w(p))$ for $\alpha=1.0$.}
  \label{subfig2:gradient1}
\end{subfigure}\hfil ~
\begin{subfigure}[t]{0.3\columnwidth}
   \centering 
  \includegraphics[width=\linewidth]{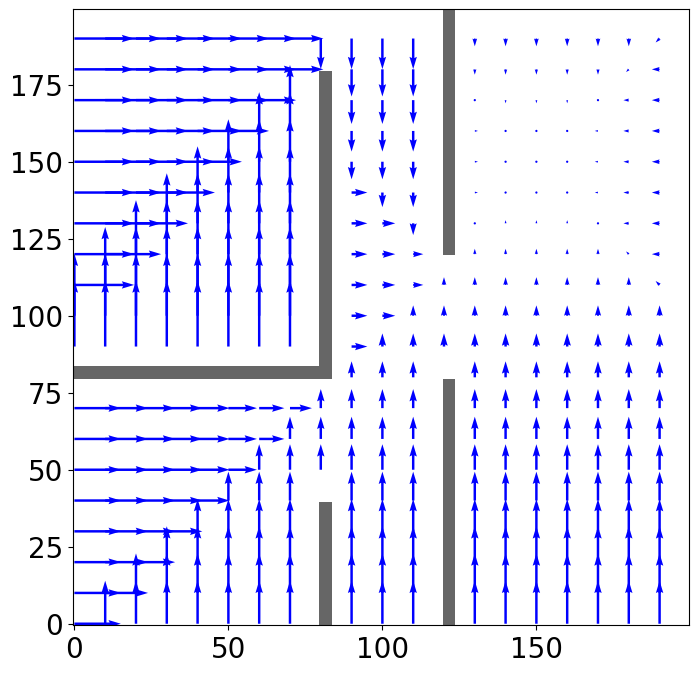}
  \caption{$\nabla \log( w(p))$ for $\alpha=2.0$.}
  \label{subfig2:gradient2}
\end{subfigure}\hfil ~
\begin{subfigure}[t]{0.31\columnwidth} 
   \centering 
  \includegraphics[width=\linewidth]{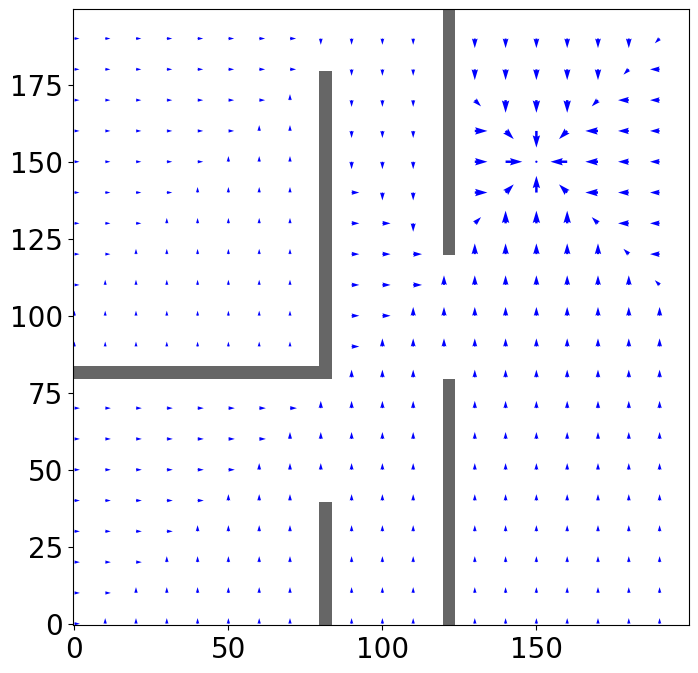}
  \caption{$\nabla \log( w(p))$ for $\alpha=0.8$.}
  \label{subfig2:gradient0.8}
\end{subfigure}\hfil 
\caption{{\small The Perceived Signal $w(p)$ maps and their corresponding gradient fields for $\beta=1$ and $\alpha\in \{ 1.0,2.0,0.8\}$ respectively. In case corresponding to $\alpha=1$ the perceived signal is the original signal but for $\alpha\in\{0.8,2\}$ the Prelec's weighing function changes the perceived signal strength and its gradients.} \vspace{-25pt}
}.
\label{fig:RAL_BEASST_SOIs_representation}
\end{figure}

To obtain the steepest descent of $\mathcal{H}_B$ within the 2-Wasserstein space, we introduce a variation in $w(\vect{x})$:
\begin{equation}
    w(\mathbf{x}) + \delta w(\mathbf{x}) = w(p(\mathbf{x})) + \epsilon \,\eta(\mathbf{x}),
\end{equation}
where $\epsilon$ is a small parameter, $\eta(\mathbf{x})$ represents the direction in which $w$ is being perturbed. The variation $\delta w(\mathbf{x}) = \epsilon \,\eta(\mathbf{x})$, acts as a test direction for computing the change in $w$ as robot moves. We then express the behavior entropy $\mathcal{H}_B$ in terms of the perturbed weighting function:
\begin{equation*}
    \mathcal{H}_B(w + \delta w) = -\int_{\mathcal{R}_{\mathbf{x}_r(t)}} \!\!\!\!\!(w(\mathbf{x}) + \delta w(\mathbf{x})) \log(w(\mathbf{x}) + \delta w(\mathbf{x})) \, \text{d}\mathbf{x}.
\end{equation*}
The functional derivative of $\mathcal{H}_B$ with respect to $w$, known as the Gâteaux derivative, can be computed as:
\begin{equation}
   \delta \mathcal{H}_B(w,\delta w) = \lim_{\epsilon \to 0} \frac{\mathcal{H}_B(w + \delta w) - \mathcal{H}_B(w)}{\epsilon}.
\end{equation}
Equivalently, this can be expressed as:
\begin{align}
    \delta \mathcal{H}_B(w,\delta w) &= \frac{\text{d}}{\text{d}\epsilon} \mathcal{H}_B(w(\mathbf{x}) + \delta w(\mathbf{x}))\bigg|_{\epsilon=0} \nonumber \\
    &= -\int_{\mathcal{R}_{\mathbf{x}_r(t)}} \left(\log(w(\mathbf{x})) + 1\right) \eta(\mathbf{x}) \, \text{d}\mathbf{x}.
\end{align}
From this, the first variation of $\mathcal{H}_B$ with respect to $w$ is:
\begin{equation}
    \frac{\delta \mathcal{H}_B}{\delta w} = -\log(w(\vect{x})) - 1.
\end{equation}
As shown in ~\cite{Villani:2009}, the steepest descent direction of $\mathcal{H}_B$ occurs along the 2-Wasserstein gradient flow, defined by $\vect{v(\vect{x})} = -\nabla \frac{\delta \mathcal{H}_B}{\delta w} = \nabla \log(w(\vect{x}))$.

This continuous flow field directly dictates the robot's motion at position $\vect{x}_r(t)$:
\begin{align}
    \dot{\vect{x}}_r(t) = \vect{v}(\vect{x}_r(t)) = \nabla \log(w(p(\vect{x}_r(t)))), \label{eq:robot_control_law}
\end{align}
where we emphasize the implicit dependence of $w$ on $p(\vect{x}_r(t))$ and fixed parameters $\alpha, \beta$. This control law drives the robot in the direction of the steepest decrease in $\mathcal{H}_B$, thereby reducing BE and guiding the robot toward the stronger signal region (i.e., the source). Crucially, the spatial gradient $\nabla \log(w(\vect{x}_r(t)))$ can be expressed as:
\begin{equation}\label{eq::gradient_flow}
    \nabla \log(w(\vect{x}_r(t))) = \frac{\beta \alpha (-\log p(\vect{x}_r(t)))^{\alpha-1}}{p(\vect{x}_r(t))} \nabla p(\vect{x}_r(t))\footnote{The spatial gradient $\nabla p(\vect{x})$ can be estimated in practice through methods like finite differences (from sequential measurements or using a local array of sensors) or local regression from collected data points. For continuous, real-time operation, employing a multi-antenna system or an array of multiple, closely spaced sensors is often ideal as it enables instantaneous gradient estimation without requiring robot movement for spatial sampling.}
\end{equation} 
This expression reveals the robot's velocity as a scaled version of the local spatial gradient of the normalized signal strength. The scaling factor, dependent on the subjective perception parameters $\alpha$ and $\beta$ and the current signal level, dynamically reshapes the robot's velocity field (Fig.~\ref{fig:RAL_BEASST_SOIs_representation}). The implications of this inherent rescaling mechanism are further discussed in Section~\ref{sec::adapt_param_tuner}.
\vspace{-0.2in}
\subsection{Performance Analysis}
\vspace{-0.05in}
To understand the mechanism for convergence to the source, consider the integrand of $\mathcal{H}_B$, which is $-w(p) \log(w(p))$. As depicted in Fig.~\ref{fig:RAL_BEASST_local_HB}, this function achieves its global minimum exclusively when the perceived signal strength $w(p)$ approaches either $0$ or $1$. While a state of uniformly low perceived signal ($w(p) \approx 0$) would technically minimize $\mathcal{H}_B$, the robot's control law, $\dot{\vect{x}}_r(t) = \nabla \log(w(\vect{x}_r(t)))$, inherently drives it towards \emph{increasing} values of perceived signal strength. This inherent directional bias compels the robot to seek regions where $w(p)$ is maximized (i.e., approaching $1$). Crucially, the condition $w(p) \approx 1$ is uniquely met at the source location, where the actual normalized signal $p(\vect{x})$ is maximized. Consequently, the continuous pursuit of this state of maximal perceived concentration naturally guides the robot directly to the signal peak, where $\mathcal{H}_B$ is minimized to zero. This approach offers a direct, data-driven solution to SS that is independent of explicit signal propagation models. Moreover, the adjustable parameters of Prelec's function provide a powerful mechanism to fine-tune the robot's perceived landscape, thereby influencing its SS dynamics and pursuit strategy. This fundamental principle, demonstrating how continuous minimization of BE drives the robot to the signal peak, forms the basis for the asymptotic convergence result formally stated in Theorem~\ref{thm:convergence}.

\begin{thm}\longthmtitle{Asymptotic Convergence to Unique Source for a class of sources with uni-modal signal strength}
\label{thm:convergence}
Consider a robot whose position $\vect{x}_r(t)$ evolves according to the control law~\eqref{eq:robot_control_law},
where $p(\vect{x})$ is the normalized measured signal strength defined according to~\eqref{eq::normal_Pdm}. Let Assumption~\ref{assump::Unimodal} hold. Then, the robot's position $\vect{x}_r(t)$ will asymptotically converge to the unique source location $\vect{x}_s$ for any $\alpha,\beta>0$.
\end{thm}
\begin{proof}
Consider the Lyapunov function candidate $V(\vect{x}_r(t)) = - \log(w(p(\vect{x}_r(t))))$. For any $\alpha,\beta>0$, we have $V(\vect{x}_s) = - \log(w(1)) = 0$. For $\vect{x}_r \neq \vect{x}_s$, Assumption~\ref{assump::Unimodal} ensures $p(\vect{x}_r) < 1$. Since $w(p)$ is strictly increasing on $(0,1]$ with $w(1)=1$, we have $w(p(\vect{x}_r)) < 1$, implying $V(\vect{x}_r) > 0$. As $\|\vect{x}_r-\vect{x}_s\|\to\infty$, $w(p(\vect{x}_r))\to 0^+$ and $V(\vect{x}_r)\to +\infty$, confirming $V$ is radially unbounded. 

The time derivative along trajectories is 
\begin{align*}\dot{V}(\vect{x}_r(t)) = &-\nabla \log(w(p(\vect{x}_r(t)))) \cdot \nabla \log(w(p(\vect{x}_r(t)))) \\=& -\|\nabla \log(w(p(\vect{x}_r(t))))\|^2 \leq 0\end{align*}
By LaSalle's Invariance Principle~\cite{HKK:02}, we analyze where $\dot{V} = 0$, i.e., where $\nabla \log(w(p(\vect{x}_r))) = \vect{0}$.
From~\eqref{eq::gradient_flow} we have $\nabla \log(w(p(\vect{x}_r))) = \frac{\beta \alpha (-\log p(\vect{x}_r))^{\alpha-1}}{p(\vect{x}_r)} \nabla p(\vect{x}_r)$. This equals zero under two possibilities:
\begin{enumerate}
    \item $\nabla p(\vect{x}) = \vect{0}$: For a \emph{unimodal} function like $p(\vect{x})$ with a single peak, the condition $\nabla p(\vect{x}) = \vect{0}$ implies that $\vect{x}$ is a critical point. Due to the assumption that $p(\vect{x})$ has a \emph{unique global maximum} and \emph{diminishes away from it}, the only critical point that also represents the maximum signal is precisely the source location $\vect{x}_s$. 
    \item The scaling factor $\frac{\beta \alpha (-\log p(\vect{x}_r))^{\alpha-1}}{p(\vect{x}_r)}$ is zero: Given that $\beta > 0$, $\alpha > 0$, and $p(\vect{x}_r) \in (0,1]$, this factor can only be zero if $(-\log p(\vect{x}_r))^{\alpha-1} = 0$. This implies $-\log p(\vect{x}_r) = 0$, which means $p(\vect{x}_r) = 1$. By the assumption of a \emph{unique peak value of 1 and diminishing behavior}, this condition $p(\vect{x}_r) = 1$ is satisfied \emph{only} at the source location $\vect{x}_s$.
\end{enumerate}
Thus, the largest invariant set where $\dot{V} \!=\! 0$ is $\{\vect{x}_s\}$. By LaSalle's Invariance Principle~\cite{HKK:02}, $\vect{x}_r(t) \to \vect{x}_s$ as $t \to \infty$.
\end{proof}
The next result considers the impact of imperfect gradient measurements and external disturbances on the source-seeking performance. Such disturbances are inevitable in practice due to sensor noise, communication delays, actuator limitations, and environmental factors. We model these imperfections as bounded additive disturbances and establish practical stability guarantees.
\begin{lem}\longthmtitle{Robustness to Bounded Disturbances}
\label{lem:robustness}
Consider a robot whose position $\vect{x}_r(t)$ evolves according to
\begin{align}
\dot{\vect{x}}_r(t) = \nabla \log(w(p(\vect{x}_r(t))))+\vect{\omega}(t), \label{eq:robot_control_law_noise}
\end{align}
where $\vect{\omega}(t)$ is a bounded disturbance satisfying $\|\vect{\omega}(t)\| \leq \bar{\omega}$ for some $\bar{\omega} \geq 0$.
Let $p(\vect{x})$ be the normalized signal strength defined in ~\eqref{eq::normal_Pdm}, and let Assumption~\ref{assump::Unimodal} (unimodal signal with unique maximum) hold.
Then, for any $\alpha,\beta>0$, the robot's position $\vect{x}_r(t)$ will ultimately converge to and remain within a compact ball of $\mathcal{B}_{R(\bar{\omega})}(\vect{x}_s)$ of radius $R(\bar{\omega})$ centered at the source location $\vect{x}_s$, where $R(\bar{\omega}) \to 0$ as $\bar{\omega} \to 0$.
\end{lem}

\begin{proof}
{Consider the Lyapunov function $V(\vect{x}_r) = -\log(w(p(\vect{x}_r)))$ as in the noisy-free case from Theorem~\ref{thm:convergence}. Differentiating $V$  along trajectories ~\eqref{eq:robot_control_law_noise} gives:}
\begin{align}
\dot{V} &= \nabla V \cdot \dot{\vect{x}}_r = -\|\vect{g}(\vect{x}_r)\|^2 - \vect{g}(\vect{x}_r) \cdot \vect{\omega}
\end{align}
where $\vect{g}(\vect{x}_r) = \nabla \log(w(p(\vect{x}_r)))$. Applying the Cauchy-Schwarz inequality and Young's inequality we obtain:
\begin{align}
\dot{V} \leq -0.5\,\|\vect{g}(\vect{x}_r)\|^2 + 0.5\,\bar{\omega}^2. \label{eq:lyap_ineq}
\end{align}
Under Assumption~\ref{assump::Unimodal}, $p(\vect{x})$ has a unique global maximum at $\vect{x}_s$, which implies $\nabla p(\vect{x}) \neq 0$ for all $\vect{x} \neq \vect{x}_s$. Since $w(p)$ is strictly increasing with $w'(p) > 0$ for $p \in (0,1)$, {the composite gradient} $\vect{g}(\vect{x}_r) = \frac{w'(p(\vect{x}_r))}{w(p(\vect{x}_r))} \nabla p(\vect{x}_r) \neq 0$ for all $\vect{x}_r \neq \vect{x}_s$.
By continuity of $\vect{g}(\vect{x}_r)$ and compactness arguments, for any $\delta > 0$, there exists $\gamma(\delta) > 0$ such that
$\|\vect{g}(\vect{x}_r)\|^2 \geq \gamma(\delta)$, whenever $ \|\vect{x}_r - \vect{x}_s\| \geq \delta$. Moreover, $\gamma(\delta)$ is increasing in $\delta$ and $\gamma(\delta) \to 0^+$ as $\delta \to 0^+$. Substituting into~\eqref{eq:lyap_ineq}, we~obtain:
\begin{align}
\dot{V} \leq -0.5\,\gamma(\|\vect{x}_r - \vect{x}_s\|) + 0.5\,\bar{\omega}^2.
\end{align}
{This has the standard ISpS structure:}
\begin{align}
\dot{V} \leq -\gamma_1(\|\vect{x}_r - \vect{x}_s\|) + \gamma_2(\|\vect{\omega}\|)
\end{align}
where $\gamma_1(s) = \frac{\gamma(s)}{2}$ and $\gamma_2(s) = \frac{s^2}{2}$ are class $\mathcal{K}$ functions. Since $V$ is radially unbounded (as established in Theorem~\ref{thm:convergence}), this establishes Input-to-State Practical Stability (ISpS)~\cite{Sontag:2008}. Define $R(\bar{\omega}) = \inf\{\delta \geq 0 : \gamma(\delta) \geq \bar{\omega}^2\}$. By ISpS theory, the trajectory ultimately enters and remains within the compact ball $\mathcal{B}_{R(\bar{\omega})}(\vect{x}_s)$ of radius $R(\bar{\omega})$ centered at $\vect{x}_s$. Finally, since $\gamma(\delta) \to 0^+$ as $\delta \to 0^+$, for any $\bar{\omega} > 0$, we can find arbitrarily small $\delta > 0$ such that $\gamma(\delta) < \bar{\omega}^2$, which implies $R(\bar{\omega}) \to 0$ as $\bar{\omega} \to 0$.
\end{proof}

\begin{figure}[tbp]

    \centering
    \begin{subfigure}[t]{\columnwidth}
   \centering 
  \includegraphics[width=0.90\linewidth]{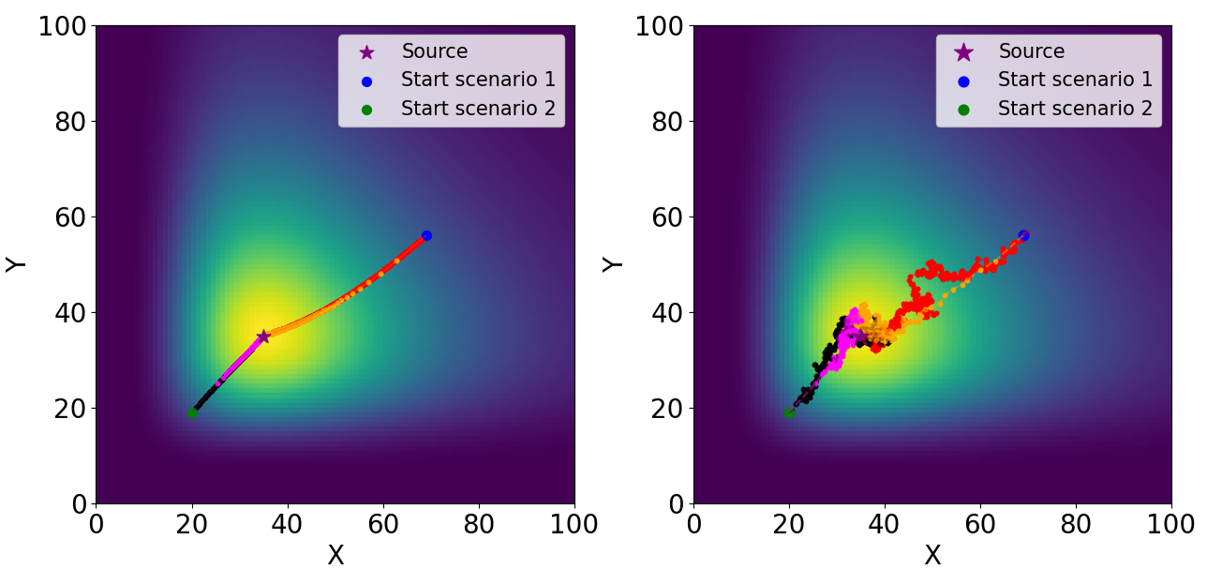}
  \caption{{\small The actual Signal Strength}}
  \label{subfig4:convergence_SF}
\end{subfigure}
\medskip
\begin{subfigure}[t]{\columnwidth}
   \centering 
  \includegraphics[width=0.9\linewidth]{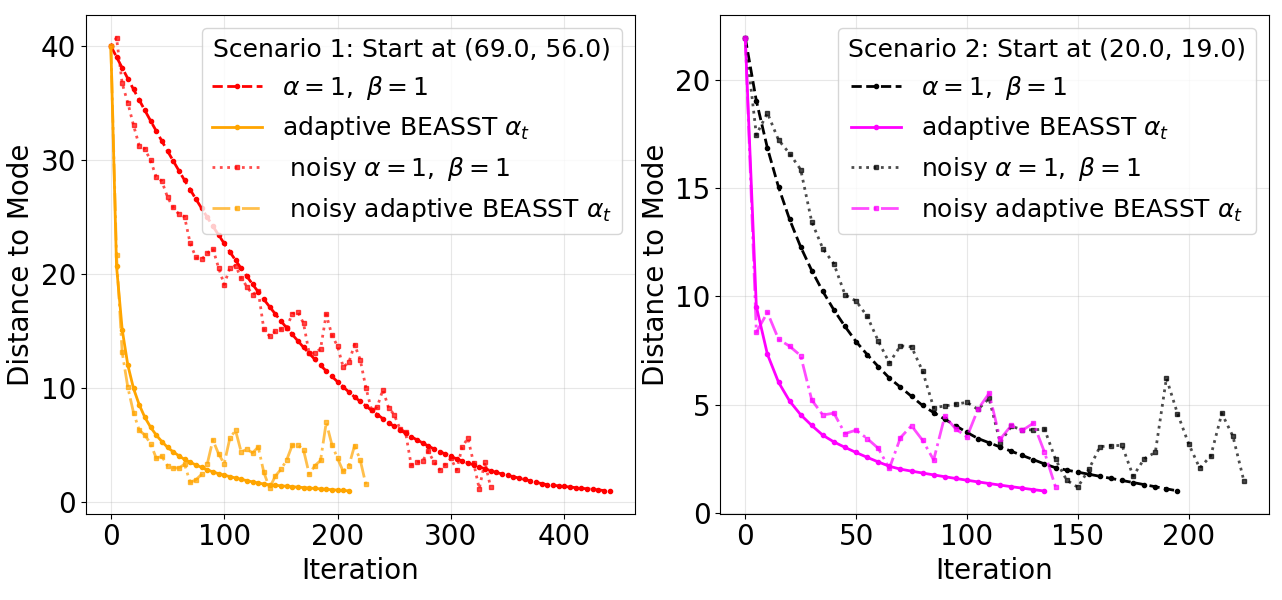}
  \caption{Convergence plot}
  \label{subfig4:convergence}
\end{subfigure}\hfil 
\vspace{-0.1in}
    \caption{{\small Comparison of convergence steps for BEASST with adaptive tuning, and its noisy variant versus fixed $\alpha=\beta=1$, which corresponds to use of SE from two distinct starting positions( Scenario 1 and 2).
    }}
    \label{fig:RAL_BEASST_shannon_slow_convergence}
\end{figure}\vspace{-0.1cm}
\subsection{Adaptive Parametric Tuner}
\label{sec::adapt_param_tuner}
\vspace{-0.05in}
As evident in~\eqref{eq::gradient_flow}, the robot's velocity field is scaled by parameters $\alpha$ and $\beta$. We set $\beta=1$ as it serves merely as a uniform scaling factor. In contrast, we leverage $\alpha$, which dynamically reshapes this gradient profile and determines the robot's search behavior. Specifically, $\alpha$ critically influences how the robot perceives and responds to signal strength: when $\alpha > 1$, the gradient magnitude is amplified in regions far from the source (where $p$ is low), promoting aggressive, rapid movement towards potential signal peaks and prioritizing broader exploration. Conversely, for $\alpha < 1$, the gradients become more pronounced near the source (where $p$ is high), encouraging cautious, precise movement essential for fine localization and preventing overshoot. These effects on the perceived signal landscape and corresponding gradients are visually demonstrated in Fig.~\ref{fig:RAL_BEASST_SOIs_representation}.

Motivated by this, we propose a dynamic tuning mechanism for $\alpha$ based on perceived signal strength:
\begin{align}
\alpha_{t}=\alpha_{max}-(\alpha_{max}-\alpha_{min})\cdot\varrho(p_{t}(\vect{x})), \label{eq::adapt_alpha}
\end{align}
where $\varrho(p(\vect{x}))$ is a smooth transition function (e.g., a logistic function) that maps the normalized signal strength $p(\vect{x})$ to a value between 0 and 1. This adaptive tuning allows the robot to seamlessly transition its behavior based on its current proximity to the source.

The crucial reshaping capability and tuning power are lost with Shannon differential entropy ($\alpha=\beta=1$ in BE $\mathcal{H}_B$). In this case, the control law~\eqref{eq:robot_control_law} simplifies to $\dot{\vect{x}}_r(t) = \frac{1}{p(\vect{x}_r(t))} \nabla p(\vect{x}_r(t))$. This non-adaptive approach leads to significantly slower convergence, especially from areas with low signal strength, unlike our method which enables effective SS even with weaker signals. This property of the tuning mechanism is beneficial for SS in unknown environments, as it expedites the process by allowing a lower signal detection threshold for transitioning from exploration to source~seeking. This effect is clearly illustrated in two scenarios in Fig.~\ref{fig:RAL_BEASST_shannon_slow_convergence} where the underlying probability field is modeled by a log-normal distribution and the convergence to the source is faster using our adaptive algorithm. We choose log normal in this example to demonstrate the effectiveness of our algorithm even in asymmetric signals that closely match real SoIs, such as acoustic signals. Furthermore, Fig.~\ref{fig:RAL_BEASST_shannon_slow_convergence} also shows that BEASST maintains convergence under bounded Gaussian noise, demonstrating robustness of the behavioral entropic gradients and confirming the theoretical stability guarantees established in Lemma~\ref{lem:robustness}.

\begin{algorithm}[t]
{\small
\caption{BEASST}
\label{alg:BEASST_multiple_sources}
\SetAlgoLined
Input: Occupancy map $\mathcal{M}_D$, measured signal strength $P_{dBm}(\vect{x})$, its gradient $\nabla P_{dBm}(\vect{x})$, number of sources $N_s$; \\
Output: Robot waypoints $\{\vect{x}_{t+1}\}$; \\
Initialize: {$N_{\text{detected}} = 0$}, $\vect{x}_0$, $\alpha,\beta=1$,$\tau$,$t=0$;

\While{$N_{\text{detected}} < N_s$}{

    $P_{dBm}(\vect{x}_t),\nabla P_{dBm}(\vect{x}_t) \leftarrow \text{SourceSensing}(x_t)$;\\
    {$p_t(\vect{x}) \gets \phi(P_{dBm}(\vect{x}_t))$; \tcp{\scriptsize Normalize using Eq. \ref{eq::normal_Pdm}}}

    \If{$p_t(\vect{x}) > \tau$}{
        $\alpha_{t}=\alpha_{max}-(\alpha_{max}-\alpha_{min})\cdot\varrho(p_{t}(\vect{x}))$; \\
         $\nabla \log(w(p_t(\vect{x}))) \gets \text{Gradient}(x_t)$; \tcp{\scriptsize Eq.\ref{eq::gradient_flow}} 
        
        $\vect{x}_{t+1} \gets \vect{x}_t + \gamma \nabla \log(w(p_t(\vect{x})))$; 
        
        \If{Robot reaches a source location (e.g., $\vect{x}_t \approx \vect{x}_s$)}{
            {$N_{\text{detected}} \gets N_{\text{detected}} + 1$;} \\

{$(\mathcal{S}_{\text{pending}},  \mathcal{S}_{\text{detected}}) \leftarrow
  \text{UpdateSourceSets}(\mathcal{S}_{\text{pending}}, \mathcal{S}_{\text{detected}}, \vect{x}_t)$ }\;
        }
    }{
        $ p(occ(\mathbf{x})) \to \mathcal{N}(\mu(\vect{x}), \sigma^2(\vect{x}))$; \tcp {\scriptsize GP map} 
        
        $\vect{x}_{t+1} \gets \text{GPBasedExp}(\mathcal{M}_D, \text{p(occ(x))})$; \\
    }{$\mathcal{M}_D$ $\gets \text{Update}$ ($\mathcal{M}_D$,$\text{p(occ(x))}$); \tcp {\scriptsize SLAM Update~\cite{CN-JR-HK:14}}}
}}
\end{algorithm}

\begin{figure}[tbp]
\centering
\includegraphics[width=0.40\textwidth]{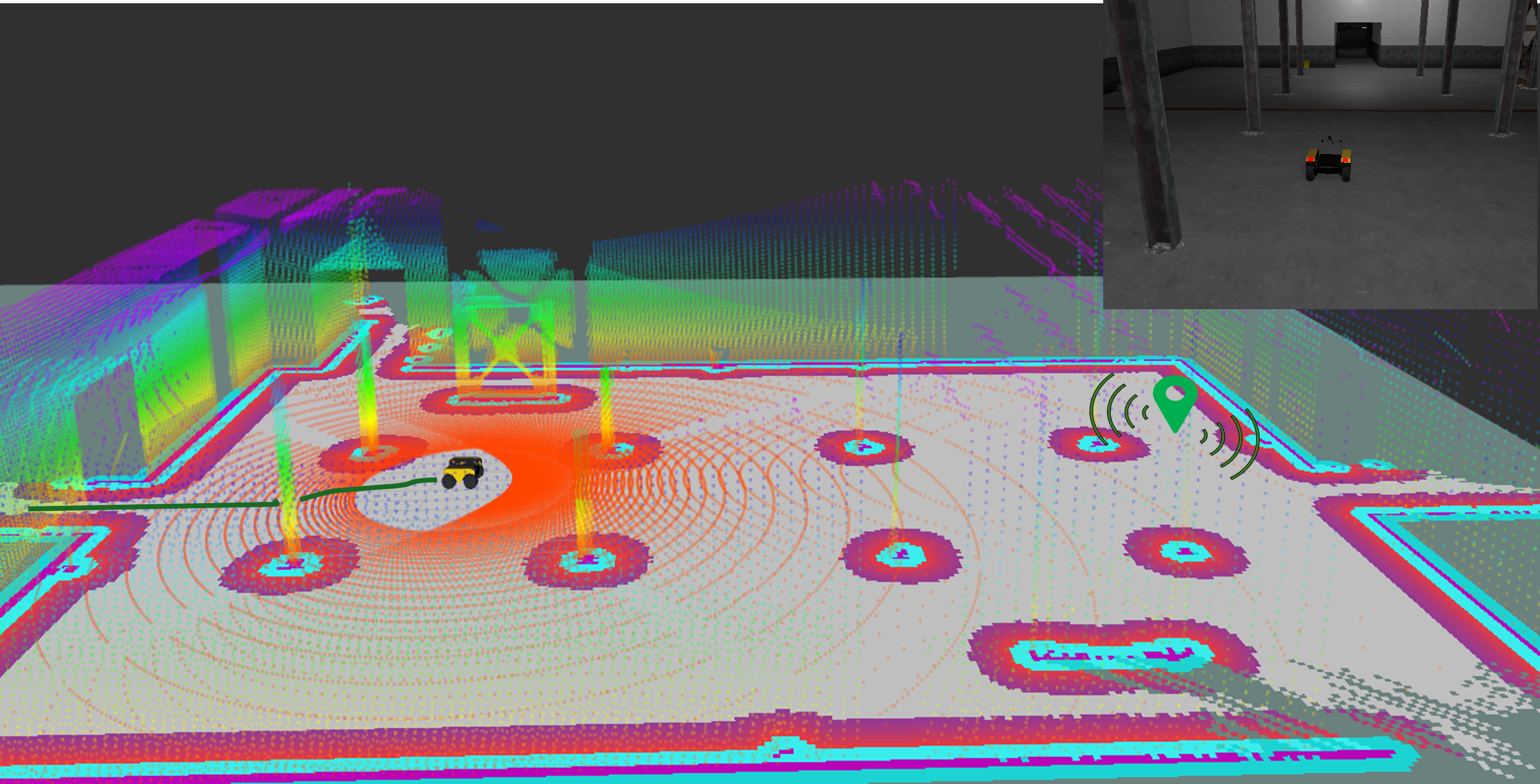}
  \caption{{\small An example Rviz visualization for Subterranean (SubT) Unity Environment. 
  The robot must autonomously navigate and search for sources in the environment, represented by green pin.
  }} \label{fig::RAL_BEASST_unity}
  \vspace{-0.1in}
\end{figure}
\begin{figure}[tbp]
\centering

\begin{subfigure}[t]{0.32\columnwidth}
 \centering
 \includegraphics[width=\linewidth]{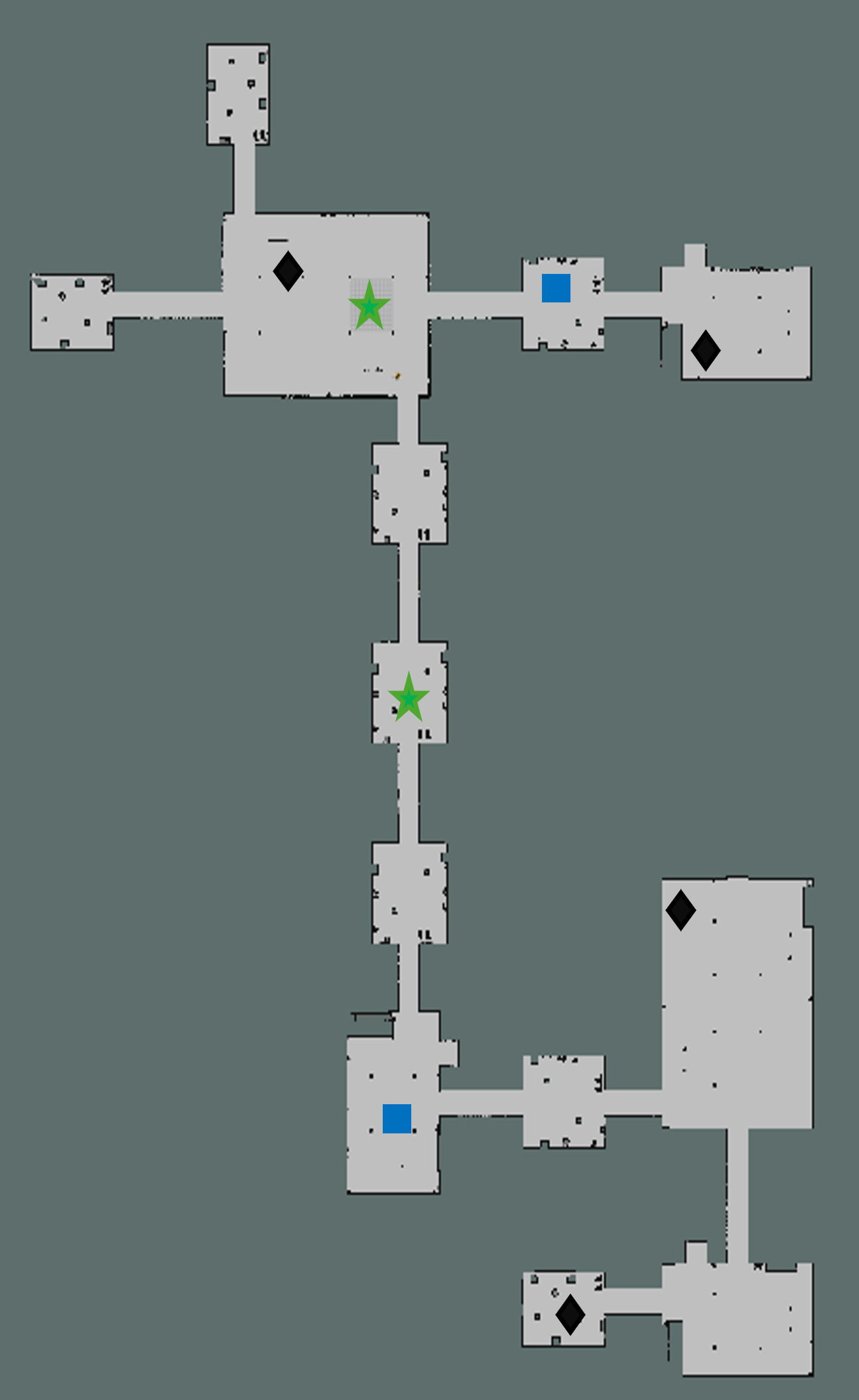} 
\caption{\small DARPA SubT map ($200\text{m} \times 302\text{m}$).  }
\label{fig4:gt_uc1}
 \end{subfigure}
 \hfill
 \begin{subfigure}[t]{0.32\columnwidth}
\centering
 \includegraphics[width=\linewidth]{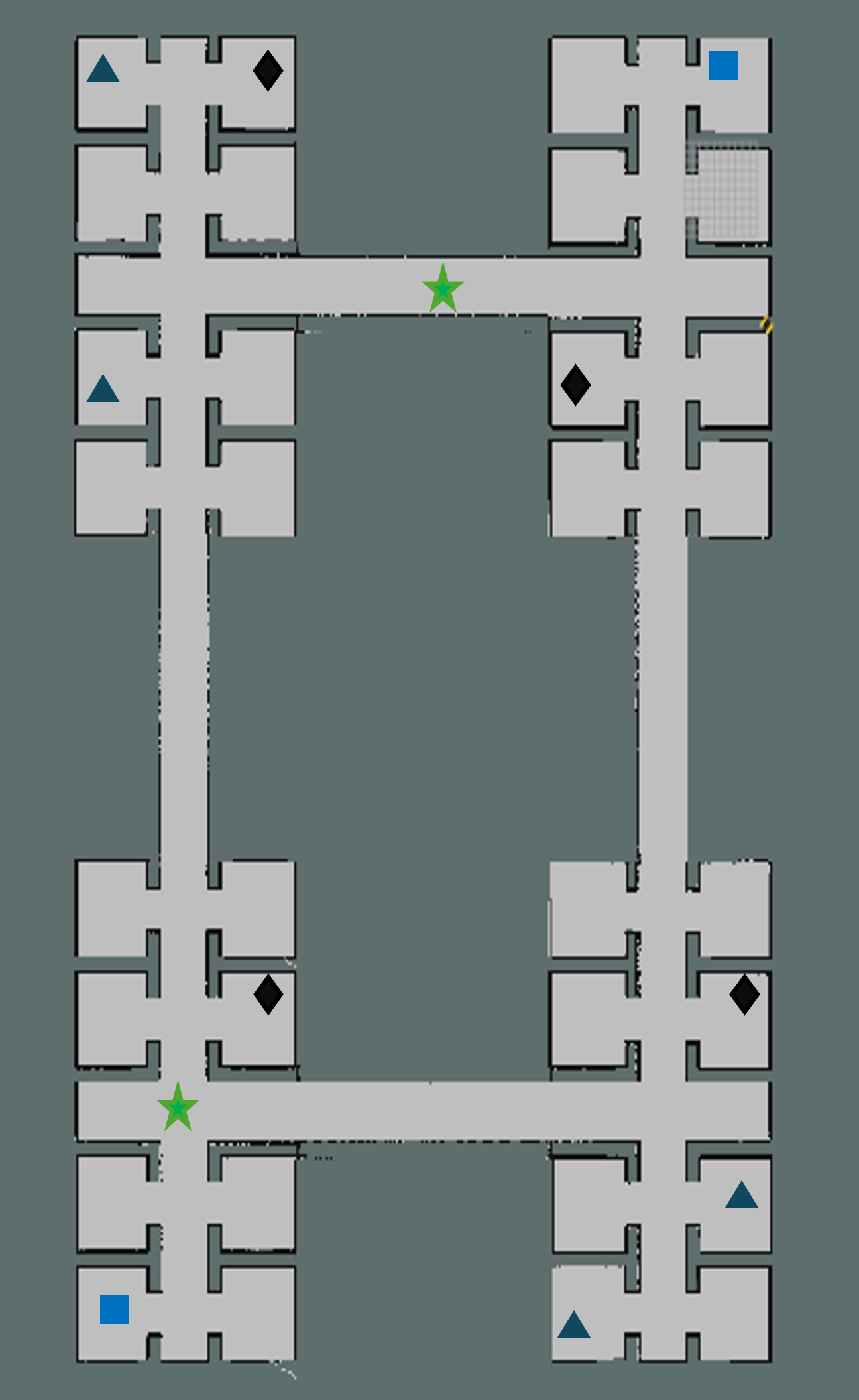} 
 \caption{\small $32$-room map. ($125\text{m} \times 171\text{m}$) }
 \label{fig4:gt_32room}
 \end{subfigure}
 \hfill
 \begin{subfigure}[t]{0.32\columnwidth}
\centering
 \includegraphics[width=\linewidth]{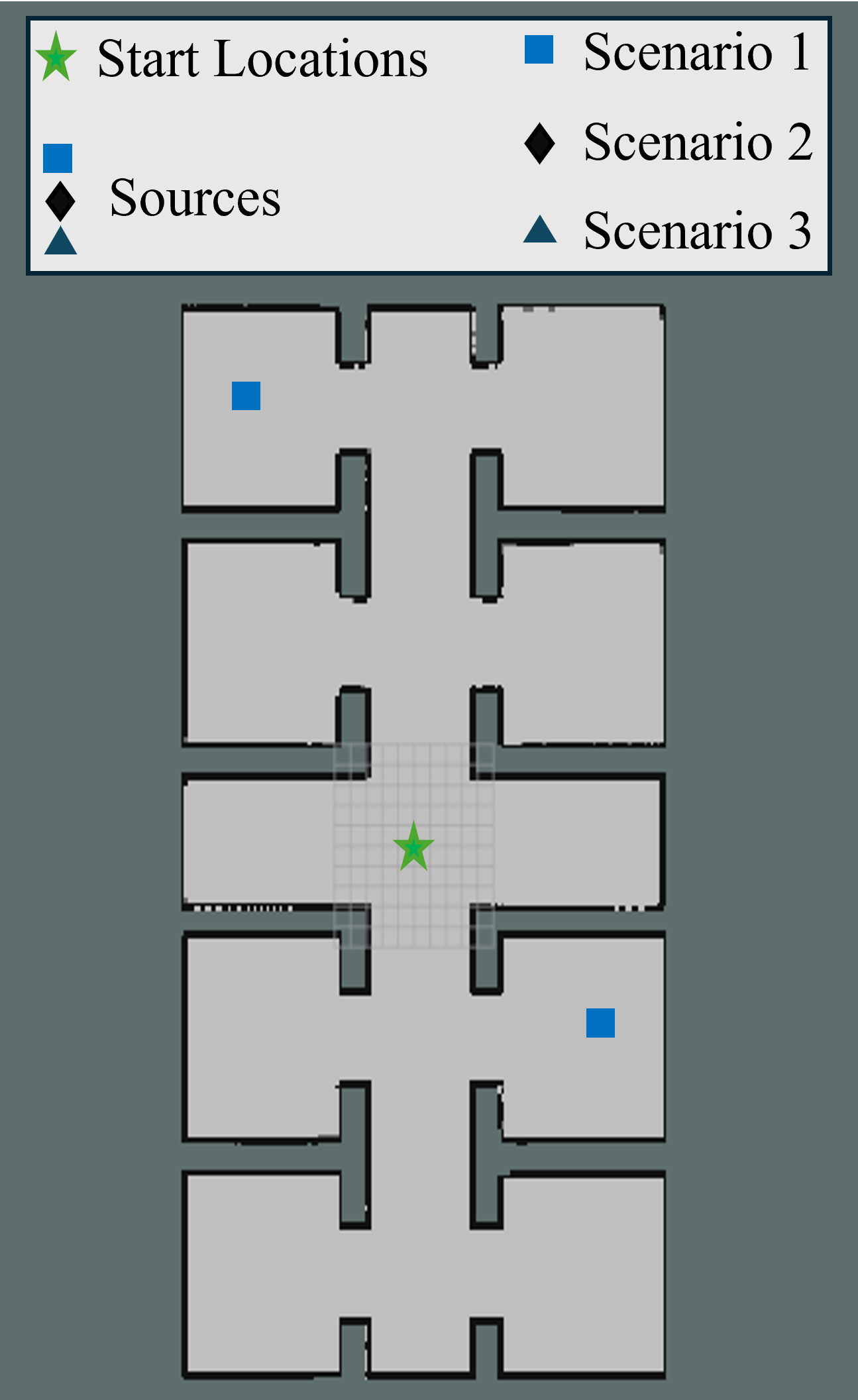} 
 \caption{\small $8$-room map  ($57\text{m} \times 78\text{m}$) }
 \label{fig4:gt_8room}
 \end{subfigure}
\caption{{\small Subfigures (a)-(c) illustrate the complete map DARPA SubT Challenge, 32-room, and 8-room environments and their scale in meters(m). These environments include different source locations and configurations used in our experiments, where each scenario is tested over ten trials (see Fig.~\ref{fig::RAL_BEASST_histogram_scenarios}}).}
 \label{fig:RAL_BEASST_UnityEnvironments_gt}
\end{figure}
\section{ Behavioral Source Seeking with Exploration}
\label{sec::BEASST_framework}
\vspace{-0.05in}
Our BEASST framework combines exploration and SS, allowing the robot to dynamically prioritize its objectives. The complete behavioral logic is encapsulated in Algorithm~\ref{alg:BEASST_multiple_sources}.
The robot's exploration phase initiates with the construction of a continuous map from its occupancy grid map $\mathcal{M}_D$ using a GP Process. The predicted occupancy probability at any point $\mathbf{x}\in\mathcal{E}$ is modeled by a Gaussian distribution $
p(\text{occ}(\mathbf{x})) = \mathcal{N}(\mu(\mathbf{x}), \sigma^2(\mathbf{x}))$,
where $\mu(\mathbf{x})$ and $\sigma^2(\mathbf{x})$ represent the mean and variance of the GP at $\mathbf{x}$, respectively. Following~\cite{MA-HJ-NR-LL:23}, the GP-based exploration operates as a frontier-based strategy, enabling the robot to navigate towards regions offering the highest informational gain, rather than merely the closest frontier represented. This function identifies frontiers with maximum occupancy uncertainty by utilizing the differential entropy of the Gaussian distribution at $\mathbf{x}$, $
\mathcal{H}(\mathbf{x}) =\frac{1}{2} \log(2\,\pi \, \text{e})+\frac{1}{2} \log(\sigma^2(\mathbf{x})).$
The robot's movement during this phase is driven by an exploration utility function, GPBasedExp(.), that considers exploration uncertainty ($u_{fi}$), area of high-variance regions ($a_{fi}$), and the GP frontier direction ($\theta_{fi}$). The robot selects the GP frontier that yields the maximum exploration gain as its next navigation subgoal. If all identified GP frontiers reside within already explored (certain) locations, the robot reverts to selecting the nearest map frontier node from its metric-topological map. Subsequently, it computes the shortest path to this chosen map frontier using the A* algorithm, and the explored nodes along this path then serve as intermediate navigation waypoints. This comprehensive approach ensures efficient environmental coverage and map refinement. 

The robot initially explores the unknown environment. Upon detecting a signal, above a predefined normalized threshold $\tau>0$, it transitions to SS, driven by the gradients \eqref{eq::gradient_flow} from normalized signal strength $p_t(\vect{x})$. This operation is performed by SourceSensing(.) function in Algorithm~\ref{alg:BEASST_multiple_sources}. The gradient is computed using Gradient(.), which naturally guides the robot towards the strongest perceived signal using adaptive $\alpha_{t}$. For practical implementation, a discrete-time update for the robot's position is given in Line 10 of Algorithm~\ref{alg:BEASST_multiple_sources}, where $\vect{x}(t)$ is the desired waypoint for the robot to track. {When multiple SoIs are detected, the framework explicitly maintains a set of detected sources, dividing them into two categories: pending ($S_\textit{pending}$) and detected ($S_\textit{detected}$). Detected but unverified SoIs are stored as `virtual frontiers' (pending sources) for later inspection. Once the robot verifies the currently pursued source, it marks it as detected and removes it from the active list; the corresponding region is masked in the signal map to prevent re-targeting using the UpdateSourceSets.~Meanwhile, locations of other significant but unverified signals remain stored as virtual frontiers, serving as waypoints for future exploration in areas with pending sources. This process ensures comprehensive coverage of all detected regions. If no significant signal is detected, the robot prioritizes exploration, moving towards regions of high entropy based on GP-based exploration and Omnimapper~\cite{CN-JR-HK:14} updates the map (Update Function).~This adaptive switching allows the robot to efficiently locate multiple sources.}

\vspace{-0.5 cm}

\section{Unity-ROS Simulation and Results}
\label{sec::Simulation and Results}

\begin{figure*}
 \centering
 \includegraphics[width=\linewidth]{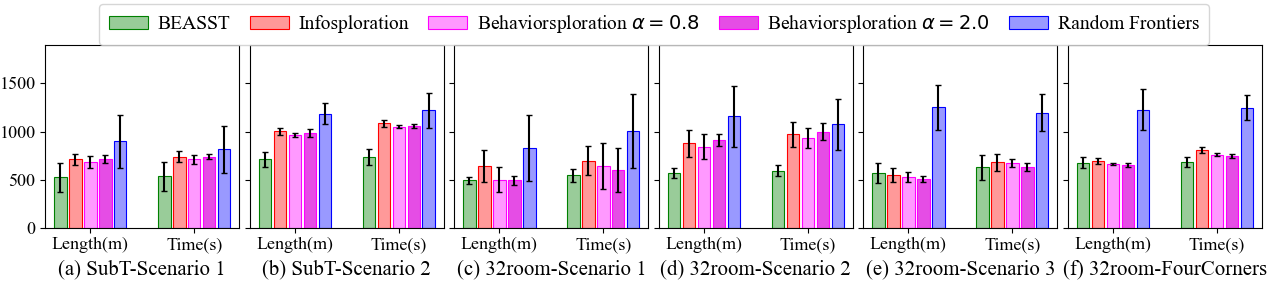}
 \vspace{-0.15in}
 \caption{{\small Path length and time taken by the robots over ten runs in different selected scenarios of DARPA SubT (Fig.~\ref{fig4:gt_uc1}) and 32 room (Fig.~\ref{fig4:gt_32room}) are shown. The histograms compare the performance of our proposed algorithm BEASST, against other frontier-based approaches such as Infosploration, Behaviorsploration($\alpha = \{0.8,2.0\}$), and Random Frontiers.}}
 \label{fig::RAL_BEASST_histogram_scenarios}
  \vspace{-0.15in}
\end{figure*}

To demonstrate BEASST's robust performance in complex environments, we conducted experiments in a Unity-ROS setup based on the DARPA Subterranean Challenge (SubT) (Fig.~\ref{fig::RAL_BEASST_unity}) and custom 8-room and 32-room Unity environments (Fig.~\ref{fig:RAL_BEASST_UnityEnvironments_gt}). This section presents key quantitative metrics, baseline comparisons, and insights from $\alpha$'s adaptive tuning.

\begin{table*}
 \centering
 \scriptsize
 \begin{tabular}{lcccc}
 \toprule
& \multicolumn{2}{c}{\textbf{Exponential Decay Model}}
 & \multicolumn{2}{c}{\textbf{Path-Loss Model}} \\
\cmidrule(lr){2-3} \cmidrule(lr){4-5}
 \textbf{Method}
 & \textbf{Length (m) ($\mu\pm\sigma$)} & \textbf{Time (s) ($\mu\pm\sigma$)}
& \textbf{Length (m) ($\mu\pm\sigma$)} & \textbf{Time (s) ($\mu\pm\sigma$)} \\
 \midrule

 \multicolumn{5}{c}{\textbf{DARPA SubT Challenge}} \\
\midrule
 \textbf{BEASST} 
& $\textbf{616.2}\pm \textbf{115.2}$ & $\textbf{634.8} \pm \textbf{115.6}$
 & $\textbf{866.3} \pm \textbf{98.6}$ & $\textbf{824.8} \pm \textbf{107.2}$ \\
{BEASST with noise}
 & ${728.3 \pm 166.2}$ & ${721.2 \pm 156.0}$
 & ${979.3 \pm 117.7}$ & ${967.6 \pm 116.1}$ \\
 Infosploration 
 & $855.2 \pm 45.3$ & $912.2 \pm 47.4$ 
& $1018.4 \pm 110.3$ & $1031.7 \pm 115.1$ \\
 Behaviorsploration ($\alpha=0.8$) 
 & $824.1 \pm 38.7$ & $879.3 \pm 32.6$ 
 & $976.1 \pm 27.8$ & $946.6 \pm 21.4$ \\
 Behaviorsploration ($\alpha=2$) 
 & $848.4 \pm 39.7$ & $898.4 \pm 25.7$ 
 & $1093.4 \pm 39.7$ & $1045.5 \pm 25.6$ \\
Random Frontiers 
 & $1041.8 \pm 196.7$ & $1172.0 \pm 194.2$
 & $1332.8 \pm 213.2$ & $1403.8 \pm 215.5$ \\
 \midrule

 \multicolumn{5}{c}{\textbf{32 Room}} \\
 \midrule
 \textbf{BEASST} 
 & $\textbf{576.2} \pm \textbf{65.1}$ & $\textbf{612.7} \pm \textbf{76.7}$
 & $\textbf{684.3} \pm \textbf{70.4}$ & $\textbf{689.7} \pm \textbf{80.9}$ \\
 {BEASST with noise}
 & ${649.2 \pm 100.7}$ & ${711.3 \pm 113.9}$
 & ${787.3 \pm 102.5}$ & ${802.5 \pm 108.5}$ \\
 Infosploration 
 & $690.1 \pm 101.2$ & $789.6 \pm 100.3$ 
 & $788.6 \pm 105.8$ & $807.5 \pm 105.7$ \\
 Behaviorsploration ($\alpha=0.8$) 
 & $631.8 \pm 79.9$ & $751.6 \pm 101.4$
 & $866.4 \pm 102.8$ & $886.3 \pm 123.8$ \\
 Behaviorsploration ($\alpha=2$) 
& $638.8 \pm 39.5$ & $744.7 \pm 94.3$ 
& $806.6 \pm 75.3$ & $825.8 \pm 122.5$ \\
 Random Frontiers 
 & $1117.7 \pm 278.0$ & $1131.1 \pm 242.6$
 & $1217.7 \pm 266.1$ & $1243.6 \pm 292.5$ \\
 \midrule

 \multicolumn{5}{c}{\textbf{8 Room}} \\
 \midrule
 \textbf{BEASST} 
& $\textbf{87} \pm \textbf{3.3}$ & $\textbf{123.1}  \pm \textbf{19.6}$ &
 $166.1 \pm 10.1$ & $184.3 \pm 23.2 $ \\
 Infosploration 
 & $143.4\pm 14.3$  & $190.2 \pm 14.4$ 
 & $196.21 \pm 19.2$ & $213.8 \pm 25.4$ \\
 Behaviorsploration ($\alpha=0.8$) 
& $95.4 \pm 9.0$  & $139.5 \pm 11.1$ 
 & $166.1 \pm 25.3$ & $184.4 \pm 23.6$ \\
Behaviorsploration ($\alpha=2$) 
 & $142.3 \pm 22.3$ & $177.4 \pm 24.9$ 
 & $\textbf{151.8} \pm \textbf{28.3}$ & $\textbf{156.6} \pm \textbf{32.5}$ \\
 Random Frontiers 
 & $181.03 \pm 61.98$ & $230.6 \pm 49.3$
 & $308.6 \pm 51.4$ & $305.6 \pm 49.2$ \\
 \bottomrule
 \end{tabular}
 
 \caption{{\small The complete summarized statistics for all the different scenarios Fig.~\ref{fig:RAL_BEASST_UnityEnvironments_gt}}. We show the average total path length and time taken over ten runs for locating sources in different scenarios with different source signal models. {The proposed policy (BEASST) and its noisy variant are compared against various methods like Infosploration, Behaviorsploration ($\alpha = \{0.8,2.0\}$), and Random Frontiers.}}
\label{tab:RAL_BEASST_Comparison_with_different_methods}
\vspace{-0.1in}
\end{table*}

\textit{Environment and Robot Setup:} Simulations used a Clearpath Warthog robot, with an Ouster LiDAR and IMU modeled to real hardware specifications. Omnimapper~\cite{CN-JR-HK:14} was used for robot localization ($\vect{x}$) and continuous occupancy map ($\mathcal{O}$) updates. We tested the algorithm across diverse environments to evaluate its source-seeking ability, with the number of sources ($N_s$) predefined per scenario. In source-seeking, robot performance depends on the strength of detected SoIs, often a function of distance. For instance, RF signals follow the Log-Distance Path-Loss model~\cite{LC-JE-AA:22}:
\begin{equation}\label{eq::Log-Distance Path-Loss}
P_{\text{dBm}}(\vect{x}) = L_0 - \underbrace{10n \log_{10}(\|\vect{x}-\vect{x}_s\|)}_{\text{Path Loss}}
- \underbrace{f(\vect{x},\vect{x}_s)}_{\text{Shadowing}}
- \!\!\!\!\!\underbrace{\epsilon}_{\text{Multipath Fading}}
\end{equation}
Here, $L_0$ is the reference power, $n$ is the decay exponent, $f(.,.)$ models shadowing, and $\epsilon$ models multipath fading. This model closely captures the real-world radio-frequency(RF) signal propagation, accounting for common environmental interference such as obstacles and reflective surfaces, enhancing the realism of our simulation. Alternatively, an exponential decay model~\cite{BB-NC-AI:16} can be used directly as a signal strength map (visualized in Fig.~\ref{fig:RAL_BEASST_SOIs_representation}):
\begin{equation}
\label{eq::SOI_model}
\phi(\vect{x}) = {e^{-\kappa \,\tilde{d}(\vect{x},\vect{x}_s)}}
\end{equation}
where $\tilde{d}(\vect{x},\vect{x}_s) = -\log\left(\sum_{i=1}^{N} e^{-d_{sp}(\vect{x}, \vect{x}_{s_i})}\right)$ is a soft minimum approximation of the shortest path distance $d_{sp}(\vect{x}, \vect{x}_{s_i})$ between $\vect{x} \in \mathcal{E}$ and the $i$-th source $\vect{x}_{s_i}$. This model provides an abstraction beneficial for evaluating navigation strategies as it naturally embeds obstacle avoidance. We normalize these signal strength maps as shown above to use them as a surrogate probability density function p(\vect{x}). {The BEASST-generated behavioral gradient field provide high-level waypoints, which are tracked using an MPPI~\cite{GW-PD-ET:16} controller with the system model of our Clearpath Warthog platform, enforcing non-holonomic kinematics, acceleration limits, and obstacle avoidance costs.}


\textit{Results:} We quantitatively evaluate BEASST's performance across different environments. Our framework successfully completes source-seeking missions by efficiently balancing source-seeking and exploration, outperforming state-of-the-art methods. Before delving into the results, we want to explain Behaviorsploration, a variation of Infosploration~\cite{CN-JR-CH:21} that is used as one of our comparison baselines. In this framework, we replace Shannon entropy with BE as the information gain metric. Our objective is to examine the resulting behavior of robots while SS without any parametric tuning for $\alpha$, thereby illustrating the importance of principled tuning. Essentially, the framework uses frontier-based method and uses BE to find the most perceived informative region using a fixed $\alpha$ to locate the source. We use $\alpha = 0.8$ and $2.0$ as these values have been shown in ~\cite{AS-CN-SM:24}, to produce effective and diverse exploration behaviors. Our set of baselines therefore, includes Infosploration~\cite{CN-JR-CH:21}, Behaviorsploration (fixed $\alpha = \{0.8, 2.0\}$), and Random Frontier Based Exploration (purely random frontier selection when no signal is detected). Each environment was tested over ten trials per scenario, varying source placements, robot initial positions, and number of sources. We measured total time taken and total path length. Table~\ref{tab:RAL_BEASST_Comparison_with_different_methods} shows that in environments with consistently detectable signals (e.g., 8-room), BEASST outperforms Infosploration by reaching sources 30 to 50 seconds faster across ten runs for different models. This improvement stems from BEASST's adaptive tuning, enabling aggressive steps with reliable source signals. We set $\alpha_{max} = 2.5$ and $\alpha_{min}=0.5$ in these experiments to prevent gradient overshoot. Furthermore, in environments with intermittently available source signals (32-room, DARPA SubT), BEASST significantly outperformed all other methods by effectively balancing exploration and source seeking. Figure~\ref{fig::RAL_BEASST_histogram_scenarios} and Table~\ref{tab:RAL_BEASST_Comparison_with_different_methods} quantitatively highlight BEASST's advantages. In the SubT environment, our framework achieved an approximate 15\% reduction in total path length and located the source 180 to 200 seconds faster. This improvement is primarily due to the dynamic and adaptive reweighing of uncertainty using our adaptive tuning mechanism: quick, decisive movements in high-confidence regions and cautious, exploratory behavior in low-confidence regions. Similarly, in the 32-room environment, BEASST also demonstrated a significant reduction in time taken to locate the source and total path length. {We also assessed a noisy gradient implementation in the DARPA SubT Challenge and 32-room environments for two signal models. Specifically, we added zero-mean Gaussian noise, $\vect{\omega} \sim \mathcal{N}(\mathbf{0}, \mathbf{I})$ where $\|\vect{\omega}(t)\| \leq 1$, truncated to enforce bounded magnitude in accordance with Lemma~\ref{lem:robustness}. We see that compared to the noise-free BEASST, the noisy-gradient variant increases path length and time by $~12-15 \%$  but in average still outperforms the noise-free benchmarks ($5-11\%$) as shown in the Table~\ref{tab:RAL_BEASST_Comparison_with_different_methods}. These results show the effectiveness of our approach and are consistent with Lemma~\ref{lem:robustness}: additive gradient errors causes a performance loss but do not break convergence.}
\vspace{-0.5 cm}
\section{Conclusion and Future Work}
\vspace{-0.2 cm}
This work presented a novel gradient-based algorithm for robotic source seeking that dynamically adapts trajectory based on signal reliability using Behavioral Entropy with Prelec's probability weighting function. The approach established theoretical convergence guarantees under unimodal signals, provided practical stability under bounded disturbances, and introduced adaptive parameter tuning for autonomous exploration-exploitation transitions integrated with Gaussian Process-based exploration. Experimental validation across DARPA SubT and multi-room scenarios demonstrated that BEASST consistently outperformed state-of-the-art methods, achieving 180-200 seconds faster convergence with reduced path length through intelligent uncertainty-driven navigation. {Future work will address hardware experiments incorporating source measurement noise through robust estimation and control design, incorporate robot pose uncertainty into unified localization-mapping-seeking objectives, and extend the framework to multi-agent scenarios.}

\IEEEpeerreviewmaketitle

\vspace{-0.5 cm}
\section*{Acknowledgments}
\vspace{-0.2 cm}

This work was developed by Donipolo Ghimire during a 2024 summer internship at the U.S. Army Combat Capabilities Development Command Army Research Laboratory (ARL), Adelphi, MD.
\vspace{-0.2 cm}

\bibliography{bib/alias,bib/reference}
\bibliographystyle{ieeetr}
\end{document}